\documentclass[11pt]{article}

\usepackage{fullpage}

\let\proof\relax 
\let\endproof\relax 

\usepackage{amsmath,amssymb,amsfonts}
\usepackage{bbm}

\usepackage{natbib}
 \bibpunct[, ]{(}{)}{,}{a}{}{,}%

\usepackage{rotating}
\usepackage{fancyvrb}

\usepackage{xcolor}
\usepackage[colorlinks=true, citecolor=green!60!black]{hyperref}

\usepackage{parskip}
\usepackage{array}

\usepackage[ruled,noend]{algorithm2e} 

\SetAlFnt{\small}
\SetAlCapFnt{\small}
\SetAlCapNameFnt{\small}
\SetAlCapHSkip{0pt}
\IncMargin{-\parindent}

\newcommand{\norm}[1]{\Vert #1 \Vert}
\newcommand{\abs}[1]{\vert #1 \vert}
\renewcommand{\dot}[2]{\left\langle #1, #2 \right\rangle}

\renewcommand{\P}{\mathbb{P}}

\newcommand{\Vol}{\textsf{Vol}}
\newcommand{\Reg}{\text{\upshape \textsf{Regret}}}
\newcommand{\cg}{\textsf{cg}}

\newcommand{\R}{\mathbb{R}}
\newcommand{\B}{\mathsf{B}}
\newcommand{\eps}{\varepsilon}
\newcommand{\poly}{\textsf{poly}}
\newcommand{\1}{\mathbbm{1}}
\newcommand{\sgn}{\textsf{sign}}

\newcommand{\bbR}{\mathbb{R}}

\newcommand{\thetast}{\theta^{\star}}
\newcommand{\yst}{y^{\star}}

\newcommand{\xst}{x^{\star}}

\newcommand{\bisection}{\textsc{Bisection}}
\newcommand{\cutplane}{\textsc{CutPlanes}}

\newcommand{\croco}{\textsc{CRoCO}}
\newcommand{\crocs}{\textsc{CRoCS}}

\usepackage[suppress]{color-edits}
\addauthor{cp}{red}
\addauthor{renato}{blue}
\addauthor{rpl}{blue}
\addauthor{jon}{green}

\newcommand\numberthis{\addtocounter{equation}{1}\tag{\theequation}}

\title{Density-Based Algorithms for Corruption-Robust Contextual Search and Convex Optimization\footnote{An extended abstract of this work was published at COLT22 under the title ``Corruption-Robust Contextual Search through Density Updates''. The main change in the title is that we have included a reference to ``Convex Optimization'', which is a new contribution of this work and didn't exist in the extended abstract.}}
\author{Renato Paes Leme\footnote{Google Research NYC, \texttt{renatoppl@google.com}} \and 
Chara Podimata\footnote{MIT, \texttt{podimata@mit.edu}. Part of the work was done while the author was a PhD intern at Google NYC.} \and
Jon Schneider\footnote{Google Research NYC, \texttt{jschnei@google.com}}}
\date{\today}

\newtheorem{theorem}{Theorem}[section]
\newtheorem{lemma}[theorem]{Lemma}
\newtheorem{proposition}[theorem]{Proposition}
\newtheorem{definition}[theorem]{Definition}

\newtheorem{corollary}[theorem]{Corollary}

\newcommand{\Halmos}{\blacksquare}

\begin{document}
\maketitle

\begin{abstract}
    We study the problem of contextual search, a generalization of binary search in higher dimensions, in the adversarial noise model. Let $d$ be the dimension of the problem, $T$ be the time horizon and $C$ be the total amount of adversarial noise in the system. We focus on the $\eps$-ball and the \rpledit{symmetric} loss. For the $\eps$-ball loss, we give a tight regret bound of $O(C + d \log(1/\eps))$ improving over the $O(d^3 \log(1/\eps) \log^2(T) + C \log(T) \log(1/\eps))$ bound of Krishnamurthy et al (Operations Research '23). For the \rpledit{symmetric} loss, we give an efficient algorithm with regret $O(C+d \log T)$. To tackle the \rpledit{symmetric} loss case, we study the more general setting of Corruption-Robust Convex Optimization with Subgradient feedback, which is of independent interest.

Our techniques are a significant departure from prior approaches. Specifically, we keep track of density functions over the candidate target vectors instead of a knowledge set consisting of the candidate target vectors consistent with the feedback obtained.
\end{abstract}

\section{Introduction}
Contextual search is a fundamental primitive in online learning with binary feedback with applications to dynamic pricing~\citep{kleinberg2003value} and personalized medicine~\citep{bayati2016}. In contextual search, there is a repeated interaction between a learner and nature; roughly speaking, in each round, the learner chooses an action based on contextual information that is revealed by nature and observes only a single bit of feedback (e.g., ``yes'' or ``no''). %
In the classic (i.e., realizable and noise-free) version, there exists a hidden vector $\thetast \in \R^d$ with $\norm{\thetast} \leq 1$ that the learner wishes to learn over time. Each round $t \in [T]$ begins with the learner receiving a context $u_t \in \R^d$ with $\norm{u_t}= 1$; this context is chosen (potentially) adversarially by nature. The learner then chooses an action $y_t \in \R$, learns the sign $\sigma_t = \sgn(\dot{u_t}{\thetast}-y_t) \in \{+1, -1\}$ and incurs loss $\ell(y_t, \dot{u_t}{\thetast})$. Importantly, the learner does \emph{not} get to observe the loss they incur, but only the sign $\sigma_t$. In this classic setting, a sequence of recent papers \citep{amin2014repeated, cohen2016feature, lobel2016multidimensional, Intrinsic18, liu21} obtained the optimal regret\footnote{We use the terms ``regret'' and ``total loss'' interchangeably.} bounds for various loss functions, as highlighted in Table~\ref{table:contextual-search}.
\begin{table}[htbp]
\centering
\begin{tabular}{ |l|l|l|l|} 
 \hline
 \textbf{Loss}  & $\ell(y_t, \yst_t)$ & \textbf{Lower Bound} &  \textbf{Upper Bound} \\ 
 \hline
  \hline
 $\eps$-ball & $\1 \{ \abs{\yst_t - y_t} \geq \eps \}$ & $\Omega(d \log(1/\eps))$ & $O(d \log(1/\eps))$  \citep{lobel2016multidimensional} \\
 \hline
 \rpledit{symmetric} & $\abs{\yst_t - y_t}$ & $\Omega(d)$ & $O(d\log d)$ \citep{liu21}\\
 \hline
 pricing & $\yst_t - y_t \1 \{ y_t \leq \yst_t \}$ & $\Omega(d \log \log T)$ & $O(d\log \log T + d \log d)$ \citep{liu21}\\
 \hline
\end{tabular}
\caption{Optimal regret guarantees for realizable contextual search.}\label{table:contextual-search}
\end{table}
The matching (up to $\log d$) upper and lower bounds in Table~\ref{table:contextual-search} indicate that the noise-free version of the problem is well understood. Beyond the classic setting, however, a lot of questions remain  when the feedback that the learner receives is perturbed by some type of noise (as is often the case in practical settings), i.e., the target value $\yst_t = \dot{u_t}{\thetast}$ is perturbed to  $\yst_t = \dot{u_t}{\thetast} + z_t$, where $z_t$ is a random variable modeling the added noise. Most of the literature thus far has focused on stochastic noise models 
\citep{nazerzadeh2016,cohen2016feature,javanmard2017perishability,shah2019semi,liu21,xu2021logarithmic,xu2022towards}, i.e., where $z_t$ is drawn from some prespecified distribution.

A recent trend in machine learning is the study of \textit{adversarial noise models}, often also called \textit{corrupted noise models}. In this model, most of the data follows a learnable pattern but an adversary can corrupt a small fraction of it. The goal is to design learning algorithms whose performance robustly degrades as a function of how much corruption was added to the data, e.g., in terms of the total number of corrupted rounds $C$. In the context of contextual search, this problem was first studied by \citet{krishnamurthy2023contextual}, who provided an algorithm with a regret bound of $O(d^3 \log(1/\eps) \log^2(T) + C \log(T) \log(1/\eps))$ for the $\eps$-ball loss (i.e., $\ell(y_t, \yst_t) = \1 \{ |\yst_t - y_t| \geq \eps\}$) and $O(d^3 \log^3(T) + C \log^2(T))$ for the \rpledit{symmetric} and pricing losses (i.e., $\ell(y_t, \yst_t) = |\yst_t - y_t|$ and $\ell(y_t, \yst_t) = \yst_t - y_t \1 \{y_t \leq \yst_t\}$ respectively).

In this paper, we provide new corruption-robust learning algorithms for contextual search with \emph{near-optimal regret guarantees}. Specifically, if $C$ denotes the total corruption, then we show the following:

\begin{enumerate}
    \item For the \emph{symmetric loss}, we give an efficient algorithm with regret $O(C+d \log T)$. This bound also extends to the setting where $C$ measures the total absolute corruption, i.e., $C = \sum_{t=1}^{T} |z_t|$.
    \item For the \emph{$\eps$-ball loss}, we give an algorithm with a tight regret bound of $O(C + d \log(1/\eps))$. This improves on the earlier bound of \citet{krishnamurthy2023contextual}. %

\end{enumerate}

To obtain the algorithm and the regret bounds for the symmetric loss, we investigate a more general setting with results that are of independent interest (\emph{Corruption-Robust Convex Optimization} ($\croco$)), which we introduce next.

\subsection{Corruption-Robust Convex Optimization}

In the standard problem of Online Convex Optimization, there is a fixed (bounded, Lipschitz) convex function $f: K \subseteq \mathbb{R}^d \rightarrow \mathbb{R}$ that the learner wishes to learn. The learner's interaction with this function is via a first-order oracle~\footnote{In contrast, a zero-th order oracle for $f$ would return the value of $f(x_t)$.}, where the learner can query a point $x_t \in K$ in the domain of $f$ and is told a subgradient $\nabla_t$ of the function at this point $x_t$. The learner would like to use this oracle to compute the minimizer $x^*$ of $f$, and more generally would like to minimize their total regret $\sum_{t \in [T]} (f(x_t) - f(x^*))$ after $T$ queries to this oracle. Convex optimization is a fundamental and incredibly well-studied problem with many efficient algorithms -- for example, gradient descent and the ellipsoid method -- that find an approximate minimizer of $f$ while incurring small regret (see Table \ref{table:croco-comparison} for a summary of a few of these methods and their properties).

Again, our interest is in optimization settings with adversarial corruptions. In the context of convex optimization, this most naturally takes the form of an adversarially perturbed oracle, which is free to report an $\eps_t$-perturbed subgradient $\tilde{\nabla}$ (see Equation~\eqref{eq:corrupted_oracle} for a precise definition). Many algorithms for convex optimization (such as the ellipsoid algorithm) are not at all robust to corruptions, since corruptions can cause them to permanently remove the true minimizer from consideration. Other algorithms (such as standard gradient descent) are somewhat robust to corruptions, but incur $O(C + d\sqrt{T})$ total regret. In this paper, we provide the first known algorithms that incur \emph{logarithmic} regret in $T$ while only scaling linearly with the number of corruptions. Specifically, we provide an algorithm \rpledit{(the Log-Concave Density Algorithm for $\croco$, in Algorithm~\ref{alg:logconcave})} which incurs at most $O(C + d\log T)$ regret after $T$ rounds.

It is not surprising that the \rpledit{Log-Concave Density algorithm} can improve over the $O(\sqrt{T})$ regret bounds in \emph{uncorrupted} settings: for example, in one-dimension, the $\bisection$ algorithm obtains $O(1)$ regret. $\bisection$ keeps an interval $[a,b]$ and queries the gradient $f'(m)$ at the midpoint $m=(a+b)/2$. If the gradient is positive, it updates the interval to $[m,b]$ and if negative to $[a,m]$. Cutting plane algorithms ($\cutplane$) correspond to a generalization of $\bisection$ to larger dimensions and they were studied extensively in a series of works for contextual search (see Section~\ref{subsec:rel-work}); these algorithms can find an $\epsilon$-optimal solution in $\log(1/\epsilon)$ iterations. Most cutting plane algorithms (like the ellipsoid method) do not offer regret guarantees. Instead, they offer a \emph{best-iterate} guarantee, i.e., guarantees on $\min_{t \in [T]} f(x_t) - f(x^{\star})$ and this guarantee is not necessarily ``last-iterate''. In other words, there is some point for which $\min_{t \in [T]} f(x_t) - f(x^{\star})$ is small, but one cannot be certain that this is true for point $x_T$. Not only that, but also while they compute a sequence of points $\{x_\tau\}_{\tau \in [t]}$ using only gradient information, they rely on evaluating the function $f(x_t)$ on those points to choose the best. Finally, $\cutplane$ algorithms are not robust to corruptions since they permanently remove elements from the consideration set. We summarize this discussion in Table~\ref{table:croco-comparison}.

\begin{table}[htbp]
\centering
\renewcommand{\arraystretch}{1.2}
\begin{tabular}{ | l |l|l|l| } 
 \hline
   & \textsf{\textbf{GD}} & \textsf{\textbf{CutPlanes}} &  \textsf{\textbf{Log-Concave Density}} {\bf \textcolor{red}{[this paper]}}\\ 
 \hline
  \hline
{\bf Regret guarantee} & $O(\sqrt{T})$  & unknown & $O(\log(T))$  \\ 
{\bf Best-iterate guarantee} & $1/\sqrt{T}$ & $\exp(-T)$ & $\log(T)/T$ \\
{\bf Uses 0-th order oracle} & No & Yes &  No \\ 
{\bf Robust to corruptions} & Yes & No  &  Yes \\
 \hline
\end{tabular}
\caption{Optimal regret guarantees for corruption-robust convex optimization.}\label{table:croco-comparison}
\end{table}

\subsection{Overview of Results and Techniques}

At the heart of our contributions lies a new family of algorithms that we introduce for contextual search. These algorithms are \emph{fundamentally different} \rpledit{from} the approach followed by \emph{every single} prior paper on contextual search, i.e., \cite{cohen2016feature, lobel2016multidimensional, Intrinsic18, liu21, krishnamurthy2021contextual}. Specifically, the algorithms in those papers keep track of a ``\emph{knowledge set}'', which is the set of all possible values of $\theta$ that are consistent with the feedback obtained. This is particularly difficult to do with corruptions and for this reason \cite{krishnamurthy2021contextual} had to develop a sophisticated machinery based on convex geometry to certify that a certain region of possible $\theta$'s can be removed from the knowledge set. 

Instead, we develop a suite of techniques based on maintaining probability density functions over the set of possible values of $\theta$. Intuitively, the density measures to what extent a given value is consistent with the feedback obtained so far. This leads to a more \emph{forgiving} update, that never removes a value from consideration; instead, it just decreases its weight. Surprisingly, these forgiving updates, if chosen properly, can yield the very fast, logarithmic regret guarantees when $C \approx 0$. 

In Section~\ref{sec:croco}, we analyze the problem of corruption-robust convex optimization. Our work is the first to formalize and analyze this setting. To tackle this problem, we propose a new update rule for the densities that is inspired by the update rule used in Eldan's stochastic localization procedure \citep{eldan2013thin}. The advantage of this update is that the density obtained is \emph{log-concave} despite the fact that we only have first-order feedback. In our algorithm, at each round the learner chooses the \emph{centroid} of the density over the knowledge set as their query point. The fact that the density maintained at all times is log-concave allows us to compute its centroid in polynomial time \cite[Chapter~9]{lee_vempala}. Additionally, it leads to a finer control over the amount of the corruption introduced, which leads to the $C_1$-bound instead of $C$.\footnote{Roughly, $C_1$ corresponds to the total absolute corruption introduced. It is formally defined in Section~\ref{sec:model_contextual_search}.}

In Section~\ref{sec:eps-ball}, we focus on the $\eps$-ball loss. Note that this loss function (contrary to the symmetric loss) is not a subcase of corruption-robust convex optimization. At a high level, our method here relies on densities once again. To translate a density into an action $y_t$ in a given round, we introduce the notion of the \emph{$\eps$-window-median} of a distribution supported in $\R$. The $0$-window-median corresponds to the usual median, i.e., a point $m$ such that the total mass above $m$ is equal to the total mass below $m$. The $\eps$-window-median corresponds to the point $m$ such that the total mass above $m+\eps$ is equal to the total mass below $m-\eps$. Our algorithm proceeds by taking the $\eps$-window-median with respect to a projection of the density onto the given context and using it both to compute the query point $y_t$ and the density update. For the $\eps$-ball loss our regret guarantees are \emph{tight}; our algorithm incurs regret $O(C + d \log (1/\eps))$, which matches the lower bound for the corruption-robust contextual search setting (Section~\ref{sec:eps-ball}).

\subsection{Related Work}\label{subsec:rel-work}

Our work is related to two streams of literature; \emph{contextual search} and \emph{adversarial corruptions in learning with bandit feedback}.

\paragraph{Contextual Search.} The classic setting of contextual search has been extensively studied by a series of papers. The approach taken traditionally by the literature has been ``bisection-based''; specifically, the learner maintains a \emph{knowledge set} throughout the $T$ rounds, which contains the set $K$ of all possible vectors for $\thetast$ (i.e., $\thetast \in K$). After every round, the learner eliminates part of the knowledge set according to the feedback that they receive. For example, assume that for a round $t \in [T]$, the learner has observed context $u_t \in \bbR^d$, queried point $y_t \in \bbR$, and received feedback $y_t \geq \yst$ (i.e., $y_t \geq \langle u_t, \thetast \rangle$). Then, for rounds $t+1$ and onward, the learner eliminates all the vectors \rpledit{$\theta \in K$ such that $y_t < \langle u_t, \theta \rangle$}. This family of bisection-based algorithms was first introduced by \citet{cohen2016feature} and we referred to them earlier as $\cutplane$ methods. \citet{cohen2016feature}'s method obtained regret guarantees $O(d^2 \log (d / \eps))$ for the $\eps$-ball and $O(d^2 \log T)$ for the symmetric loss. Subsequently, still drawing intuition from bisection-based methods, \citet{lobel2016multidimensional} introduced the \emph{ProjectedVolume} algorithm, which obtained the optimal regret $O(d \log (1/\eps)$ for the $\eps$-ball loss and an improved regret of $O(d \log T)$ for the symmetric loss. \citet{Intrinsic18} obtained the optimal (in terms of $T$) regret of $O(d^4)$ for the symmetric loss, and \citet{liu21} obtained the optimal regret (in terms of both $T$ and $d$) of $O(d \log d)$ for the symmetric loss. The algorithms highlighted above are generally brittle to adversarial corruptions and noise. This is to be expected: if at some point the corrupted feedback makes the algorithm eliminate $\thetast$ from the knowledge set, then the algorithm will never recover and will ultimately incur linear regret.

\citet{krishnamurthy2023contextual} were the first ones to study contextual search against corrupted feedback. Their algorithm was an involved adaptation of the ProjectedVolume algorithm and achieved regret $O(d^3 \log (1/\eps) \log^2(T) + C \log (T) \log (1/\eps))$ for the $\eps$-ball loss and $O(d^3 \log^3(T) + C \log^2(T))$, where $C$ is the total number of corrupted rounds that the algorithm faces. Importantly, the algorithm of \citet{krishnamurthy2023contextual} need not know $C$ a priori. Since their algorithm was an adaptation of ProjectedVolume, they were able to provide regret bounds for the pricing loss too. \citet{krishnamurthy2023contextual}'s regret bounds were far from optimal. \rpledit{Specifically, the best known lower bound for the $\eps$-ball loss is $\Omega(C + d \log (1/\eps))$ and for the symmetric loss $\Omega(C + d)$, by combining the natural $\Omega(C)$ lower bound for corrupted settings with the lower bounds in \cite{lobel2016multidimensional} for the uncorrupted setting.}

The biggest departure of our work compared to the aforementioned literature is \emph{methodological}; our algorithms are \emph{not} using bisection-based techniques. In fact, our work is the first one in the space of contextual search that maintains (at all times) a probability distribution over the entire initial knowledge set and never eliminates any part of it. Instead, it shifts probability mass around points; higher probability mass corresponds to a point that is more consistent with the feedback that the algorithm has received thus far. This is what allows us to obtain \emph{significantly} improved bounds compared to \citet{krishnamurthy2023contextual}; in fact, for the $\eps$-ball loss, we obtain the optimal regret $O(C + d \log (1/\eps))$ and for the symmetric loss $O (C + d \log T)$. Our algorithms are not only agnostic to $C$ but they also need not know that nature may send corrupted feedback at all.

\paragraph{Adversarial Corruptions in Learning with Bandit Feedback.} To model adversarial noise, we draw inspiration from the model of \emph{adversarial corruptions}, first studied by \citet{lykouris2018stochastic} in the context of multi-armed bandits. In the original model, a learner is interacting with a set of bandits with stochastic rewards and the adversary has a total budget of $C$ corruptions, i.e., at any point during the $T$ rounds the adversary can change the reward that the learner sees and all the changes have to be at most $C$. The regret guarantees of the original paper about stochastic multi-armed bandits were later strengthened by \citet{gupta2019better}, \citet{zimmert2021tsallis} and \citet{masoudian2021improved}. The original model of adversarial corruptions has since been applied to a wide range of problems; examples include linear optimization~\citep{li2019stochastic}, assortment optimization~\citep{chen2024robust}, Gaussian bandit optimization~\citep{bogunovic2020corruption}, learning product rankings~\citep{golrezaei2021learning}, dueling bandits~\citep{agarwal2021stochastic}, and both linear and non-linear contextual bandits~\citep{he2022nearly, ye2023corruption}. 
\section{Model \& Preliminaries}

In this section, we summarize the two problems of interest: Corruption-Robust Contextual Search ($\crocs$) in Section~\ref{sec:model_contextual_search} and Corruption-Robust Convex Optimization ($ \croco$) in Section~\ref{subsec:model-croco}. We also outline some useful preliminaries for the rest of the paper. 

\subsection{Convex Sets and Density Functions}

Given a vector $v \in \R^d$ and a real number $r > 0$ we define the ball around $v$ of radius $r$ as $\B(v,r) = \{x \in \R^d; \norm{x-v} \leq r\}$ where $\norm{\cdot}$ is the  $\ell_2$ norm. We use $\Vol(\B(v,r))$ to denote the volume of $\B(v,r)$, i.e., $\Vol(\B(v,r)) = \int_{\B(v,r)} 1 dx$. We often write $\B$ to refer to the unit ball $\B(0,1)$. For a set $K \subseteq \R^d$, we define its \emph{diameter} to be the biggest distance between any two vectors in $K$ measured in terms of the $\ell_2$ norm.

\paragraph{Densities.} We say that a function $\mu:\R^d \rightarrow \R_+$ is a \emph{density} function if it is integrable and integrates to $1$, i.e., $\int_{\R^d} \mu(x) dx = 1$. We say that a random variable $Z$ is drawn from a probability distribution with density $\mu$ if for every measurable set $S \subseteq \R^d$ it holds that $\P[Z \in S] = \int_S \mu(x) dx$. Given a measurable set $S$, we refer to the function $\mu_S(x) = \1\{x \in S\} / \int_S 1 dx$ as the  uniform density over $S$. To simplify notation, we write $\mu(x)$ (instead of $\mu_S(x)$) whenever clear from context. %

\paragraph{Log-Concave Densities.} We give a brief introduction to log-concave densities, which are used in Section \ref{sec:croco}. For a more complete introduction, see the book by \cite{lee_vempala} or the survey by \cite{lovasz2007geometry}.

\begin{definition}[Log-Concave Functions]\label{def:logconcave}
A function $\rpledit{\mu}: \bbR^d \to \bbR$ is called \emph{log-concave} if it is of the form $\mu(x) = \exp(-g(x))$ for some convex function $g:\bbR^d \to \bbR \rpledit{\cup \{\infty\}}$. If $\int \mu(x) dx = 1$, then we say that $\rpledit{\mu}: \bbR^d \to \bbR$ is a \emph{log-concave density function}. 
\end{definition}
Two important examples of log-concave densities are the Gaussian density (where $g(x) = \norm{x}^2$) and the uniform over a convex set $K \subseteq \bbR^d$ where $g(x) = 0$ for $x \in K$ and $g(x) = \infty$ for $x \notin K$.

We denote by $\cg(\mu, S)$ the \emph{centroid} of $\mu$ over set $S$, defined as:
$$\cg(\mu, S) \triangleq \frac{\int_{S} x \mu(x) dx}{\int_{S} \mu(x) dx}$$
To simplify notation, we write $\cg(\mu)$ to denote the centroid of $\mu$ over all of $\bbR^d$, i.e., $\cg(\mu) = \cg(\mu, \bbR^d)$.

Note that for the \emph{uniform} density over a \emph{convex} set, the above definition corresponds to the usual notion of the centroid of a convex set.

\subsection{Setting 1: Corruption-Robust Contextual Search ($\crocs$)}\label{sec:model_contextual_search}

Let $\ell$ denote a loss function $\ell : \R\times \R \rightarrow \R$ and $\yst \in \R$ be a \emph{target} value that is originally unknown to the learner. Our results for the $\crocs$ setting will focus on two specific loss functions: (i) the $\eps$-ball loss $\ell(y,\yst) = \1 \{ \abs{y-\yst} \geq \eps\}$, which penalizes each query $y$ that is far from the target by at least $\eps$; (ii) the symmetric loss $\ell(y,\yst) = \abs{y-\yst}$, which penalizes each query proportionally to how far it is from the target. Note that $\crocs$ with symmetric loss is a subset of the $\croco$ setting, which is introduced in the next section.

\paragraph{Protocol.}
In $\crocs$, there is a repeated interaction between a learner and an adversary over $T$ rounds.  The adversary initially chooses a vector $\thetast \in \B(0,1)$ that is hidden from the learner. In each round $t \in [T]$ the following events happen:
\begin{enumerate}
    \item The adversary chooses a context %
    $u_t \in \R^d$ such that $\rpledit{\|u_t\| \leq 1}$,
    and reveals it to the learner.
    \item The adversary also selects a corruption level $z_t \in [-1,1]$, which is hidden from the learner.
    \item The learner queries $y_t \in [-1,1]$.
    \item The learner receives feedback $\sigma_t = \sgn(\yst_t-y_t) \in \{-1,+1\}$, where $\yst_t = \dot{u_t}{\thetast}+z_t$.
    \item The learner incurs (but does not observe) loss $\ell(y_t, \yst_t)$.
\end{enumerate}

We consider two different measures of the total amount of corruption added to the system (the total number of noisy queries and the total deviation due to noise):
$$C_0 = \sum_{t \in [T]} \1 \{z_t \neq 0\} \qquad \text{and} \qquad 
C_1 = \sum_{t \in [T]} \abs{z_t} $$

Our goal is to upper bound the total regret $\Reg = \sum_{t \in [T]} \ell(y_t, \yst_t)$. \rpledit{Note that the regret benchmark $\min_{y \in \R^T} \sum_{t \in [T]} \ell(y_t, \yst_t)$ is zero by taking $y_t = \yst_t$.} 
We refer to this setting as ``realizable''. We do not impose any restriction on any specific corruption levels $z_t$. Instead, our eventual regret bounds are functions of the total amount of corruption ($C_0$ or $C_1$) added over the game. Importantly, our algorithms are completely agnostic to the level of corruption introduced by the adversary. The quantities $C_0$ and $C_1$ are used in the analysis but are not used by the algorithm.

\subsection{Setting 2: Corruption-Robust Convex Optimization ($\croco$)}\label{subsec:model-croco}

Consider a fixed convex $L$-Lipschitz function $f: K \subseteq \R^d \rightarrow \R$. A \emph{first-order oracle} for $f$ takes as input a point $x_t \in K$ and returns a subgradient $\nabla_t \in \R^d$, i.e., a point such that for all $z \in K$ it holds that $f(z) \geq f(x_t) + \langle \nabla_t, z-x_t \rangle$ and $\norm{\nabla_t}_2 \leq L$. Importantly, the learner has only access to this subgradient, and \emph{not} to the functional value $f(x_t)$ (i.e., the zero-th order oracle). This is crucial for applications in contextual search and market equilibrium computation \citep{paes2020computing}.

We say that the oracle is $C$-corrupted if for each queried point $x_t \in K$ it returns a vector $\tilde \nabla_t \in \R^d$, $\norm{\tilde \nabla_t}_2 \leq L$ such that:
\begin{equation}\label{eq:corrupted_oracle}
f(z) \geq f(x_t) + \left\langle \tilde \nabla_t, z-x_t \right\rangle - \epsilon_t, \forall z\in K
\end{equation}
and $\sum_{t \in [T]} \epsilon_t \leq C, \forall z \in K$. Notice that the oracle returns only $\tilde \nabla_t$, while the values of $C$ and $\epsilon_t$ remain unknown to the algorithm. We will only use $C$-corrupted oracles in our analysis.

\paragraph{Protocol.} In the corruption-robust convex optimization problem ($\croco$), there is a repeated interaction between a learner and an adversary over $T$ rounds. The adversary chooses the fixed convex function $f$, and does not reveal it to the learner. At each round, the learner issues queries $x_t \in K$, the adversary provides feedback through a $C$-corrupted first-order oracle, and the learner suffers (but does not observe) $f(x_t) - f(\xst)$, where $\xst = \arg \min_{x \in K} f(x)$. We measure the regret as $\Reg = \max_{x^{\star} \in K} \sum_{t \in [T]} (f(x_t) - f(x^{\star}))$. The learner's goal in $\croco$ is for the algorithm to achieve sublinear-in-$T$ regret, while being agnostic to the corruption level $C$.

\section{Corruption Robust Convex Optimization}\label{sec:croco}

We start our analysis from the setting of $\croco$ and subsequently, we show how the algorithm that we propose can be applied without any change in order to obtain regret bounds for the setting of $\crocs$ with symmetric loss. The main result of the section is stated below.

\begin{theorem}[Regret of the Log-Concave Density Algorithm]\label{thm:croco_log_concave_density} Let $L$ be the Lipschitz constant of function $f$ and $D$ be the diameter of set $K$. Then, for $\gamma = 1/(3LD)$,  the regret of the Log-Concave Density Algorithm for $\croco$ is $O(C + d LD \log (T/L))$.
\end{theorem}

As we outlined in Section~\ref{subsec:model-croco}, the original $\croco$ problem is formulated for a fixed convex function $f(\cdot)$. In this section, we will solve a slight generalization where we allow different convex functions $f_t: K \subseteq \R^d \rightarrow \R$ in each round as long as they share a minimizer $\xst$, i.e, there exists $\xst \in K$ such that $\xst \in \text{argmin}_{x \in K} f_t(x), \forall t \in [T]$. 

At a high level, our algorithm (called the ``\emph{Log-Concave Density Algorithm}'') maintains a density function $\mu_t$ which  keeps track of how \emph{consistent} each point $x \in K$ is with the observations (i.e., query returns) seen so far. This is a soft version of the idea used in algorithms such as ellipsoid or centroid, which remove from the consideration set all points that cannot be the minimizer given the gradient information at a given round. Since our gradients may be corrupted, instead of removing a point from consideration immediately after we obtain feedback that is inconsistent with it, we instead decrease its density.

It will be useful to keep a ``structured'' density function $\mu_t$. Formally, we make sure that $\mu_t$ is a \emph{log-concave density}. We then update it in such a way that the density $\mu_t$ remains log-concave throughout the algorithm, which enables efficient sampling from it. The algorithm we use is formally defined below.

\begin{algorithm}[htbp]
\caption{Log-Concave Density Algorithm for $\croco$}
\label{alg:logconcave}
\DontPrintSemicolon
\SetAlgoNoLine
Initialize $\mu_1(x)$ to be the uniform density over $K$ and $\gamma = (3LD)^{-1}$. \;
\For{rounds $t \in [T]$}{
    Query the centroid of distribution $\mu_t$, i.e., $x_t = \int_{K} x \mu_t(x) dx$.\;
    Receive feedback $\tilde \nabla_t$, and update the density as: 
    \[
    \mu_{t+1}(x) = \mu_t(x) \cdot \left(1 - \gamma   \cdot \dot{\tilde \nabla_t}{x - x_t}\right) \numberthis{\label{eq:logconcave-upd}}
    \]
}
\end{algorithm}

\noindent We first argue that distribution $\mu_t$ remains a log-concave density throughout $T$ rounds.

\begin{lemma}\label{lem:logconcave-valid-density}
Let $\gamma$ be such that $\gamma < (LD)^{-1}$ where $D$ is the diameter \rpledit{of} $K$ and $L$ the Lipschitz constant. Then, the density function $\mu_t$ maintained by Algorithm~\ref{alg:logconcave} is a log-concave density for all $t \in [T]$.
\end{lemma}

\noindent \proof{{\bf Proof.}}
We proceed with induction. For the base case, note that the lemma holds by definition, since $\mu_1$ is the uniform density. Assume now that the lemma holds for some $t = n$, i.e., $\mu_t$ is log-concave and $\int_K\mu_t(x) dx = 1$. We now focus on $\mu_{t+1}$. First, note that $\mu_{t+1}$ is non-negative, since from Cauchy-Schwarz
$\abs{\gamma \cdot  \langle \tilde \nabla_t,x-x_t \rangle } \leq \gamma \norm{\tilde \nabla_t} \cdot \norm{x-x_t} < 1 $ and so $1-\gamma \langle{\tilde \nabla_t},{x-x_t}\rangle > 0.$
To see that it integrates to $1$: 
\begin{align*}
    \int_{K} \mu_{t+1}(x)dx  &= \int_{K} \mu_t(x)dx - \gamma \cdot  \left\langle\tilde \nabla_t, \int_{K} x \mu_t(x) dx - x_t \right\rangle \\
    &= 1 - \gamma \cdot \left\langle \tilde \nabla_t, \int_{K} x \mu_t(x) dx - x_t \right\rangle &\tag{inductive hypothesis}\\
    &= 1 + 0 &\tag{definition of $x_t$}
\end{align*}
We are left to show that $\mu_t$ is a log-concave function for all rounds $g$. We use again induction. By definition, the uniform density is log-concave (base case). Assume now that $\mu_t$ is log-concave for some round $t=n$, and let us rewrite it as $\mu_t(x) = \exp ( - g_t(x) )$ for some convex function $g_t$ (Definition~\ref{def:logconcave}). Then, for round $t+1$, the density can be written as:
\[
\mu_{t+1}(x) = \exp(-g_{t+1}(x)), \qquad \text{for} \qquad g_{t+1}(x) = g_t(x) - \log \left( 1-\gamma_t \langle{\tilde \nabla_t},{x-x_t}\rangle\right)
\]
Note that $g_{t+1}$ is a convex function, since it is a sum of convex functions. As a result, by Definition~\ref{def:logconcave}, $\mu_{t+1}$ is a log-concave function. $\Halmos$
\endproof

\noindent We are now ready to prove the main result of this section.

\noindent \proof{{\bf Proof of Theorem~\ref{thm:croco_log_concave_density}.}}
We define a potential function corresponding to the total mass around $\xst$ and argue that picking up loss as a result of the queries we issue leads to concentration of measure:
$$\Phi_t = \int_{K \cap \B(\xst, r)} \mu_t(x) dx$$
for a radius $r = 1/(LT)$. %

By the guarantee of the corrupted oracle (Equation \eqref{eq:corrupted_oracle}), we observe that:
$$ -\left \langle \tilde \nabla_t, \xst-x_t \right\rangle \geq f_t(x_t) - f_t(\xst) - \eps_t $$
and since $\norm{\tilde \nabla_t}\leq L$ we have that for all points $x \in K \cap \B(\xst, r)$, it holds that:
$$ -\left \langle \tilde \nabla_t, x-x_t \right\rangle \geq f_t(x_t) - f_t(\xst) - \eps_t - Lr$$
We now can bound the change of potential as follows:
$$\Phi_t = \int_{K \cap \B(\xst, r)} \mu_{t-1}(x) \left(1-\gamma \left \langle \tilde \nabla_t, x-x_t \right \rangle \right) dx \geq \Phi_{t-1}\cdot (1+\gamma(f_t(x_t) - f_t(\xst) - \eps_t - Lr)) $$
We will now use the fact that $1-x \geq e^{-\alpha_1 x}$ and $1+x \geq e^{\alpha_2 x}$ for $x \in [0,2/3]$ and constants $\alpha_1 = \frac{3}{2} \ln 3 > 1$ and $\alpha_2 = \frac{3}{2} \ln \frac{5}{3} < 1$. We also observe that $0 \leq \gamma (f_t(x_t) - f_t(\xst)) \leq 1/3$ and $0 \leq \gamma (\eps_t + Lr) \leq 2/3$. With that, we bound the potential as follows:
$$\Phi_t \geq \Phi_{t-1} (1+\gamma(f_t(x_t) - f_t(\xst)) ((1-\gamma(\eps_t + Lr)) \geq \Phi_{t-1} \exp(\alpha_2 \gamma (f_t(x_t) - f_t(\xst)) - \alpha_1 \gamma(\eps_t + Lr))  $$
Telescoping and using the fact that densities integrate to at most $1$ we have:
\begin{equation}\label{eq:potential-last-step}
\textstyle 1 \geq \Phi_{T+1} \geq \Phi_1\exp(\alpha_2 \gamma \sum_t(f_t(x_t) - f_t(\xst)) - \alpha_1 \gamma(\sum_t \eps_t + TLr))
\end{equation}
Now, $\Phi_1 = \Vol(K \cap \B(\xst, r)) / \Vol(K) \geq (r/D)^d$ since by shrinking the set $K$ by a factor of $r/D$ around $\xst$ we obtain a set contained in $\B(\xst, r)$ hence the volume $\Vol(K \cap \B(\xst, r))$ is at least $(r/D)^d \Vol(K)$. Now, taking the logarithm on both sides of Equation~\eqref{eq:potential-last-step} and re-arranging, we obtain the following:
$$\sum_{t \in [T]} f_t(x_t) - f_t(\xst) \leq O \left(\sum_{t \in [T]} \eps_t + TLr - \log(\Phi_1)/\gamma \right) = O(C + 1 + LD d \log(LT)) \hspace{10pt} \Halmos$$
\endproof

\paragraph{Polynomial time implementation.} The computationally non-trivial step in the Log-Concave Density Update algorithm is the computation of the centroid. This problem boils down to integrating a log-concave function, since its $i$-th component is $\int x_i \mu_t(x) dx = \int \exp(\log x_i + \log \mu_t(x)) dx$. \rpledit{We observe that the log-density $\log \mu_t(x)$ and its gradient can be computed explicitly in $O(dT)$ time, since:
\begin{equation}\label{eq:log_density_f}
-\log \mu_t(x) = -\sum_{s=1}^t \log \left(1 - \gamma   \cdot \dot{\tilde \nabla_t}{x - x_t}\right)
\end{equation}
\begin{equation}\label{eq:grad_log_density_f}
-\nabla[\log \mu_t(x)] = \sum_{s=1}^t \frac{\gamma \tilde \nabla_t}{1 - \gamma   \cdot \dot{\tilde \nabla_t}{x - x_t}}
\end{equation}
}
\rpledit{Having access to a log-density oracle, it is possible to obtain an  $\eps$-additive approximation of the centroid in $O(\poly(d,1/\eps))$. The first such algorithm was given by \cite{applegate1991sampling} with the bound of $O(d^{10})$ oracle calls. Sampling from a log-concave distribution is an active area of research and recent algorithms provide much better bounds. \cite{lovasz1999hit} provides a $O(d^4)$ algorithm having only zero-order oracle call access to $f$. With a first order, \cite{dwivedi2019log} provides a $O(d^2)$ algorithm called MALA that combines a discretized Langevin Dynamic with the Metropolis-Hastings sampling. We refer to the book by \cite{lee_vempala} for algorithms with an improved running time and for a comparison of the different bounds.}

If given access to an approximate centroid, the proof of Theorem \ref{thm:log_concave_density} can be adapted as follows. Let $\tilde x_t$ be an approximate centroid of $\mu_t(x)$, i.e., the point $\tilde x_t$ is such that:
\begin{equation}\label{eq:apx-centr}
\left\|\tilde x_t - \frac{\int_{K} x \mu_t(x) dx}{\int_{K} \mu_t(x) dx} \right\| \leq \delta.
\end{equation}
Then, the update defined in Equation~\eqref{eq:logconcave-upd} no longer keeps $\mu_t$ a density, but it still keeps it an \emph{approximate} density as follows:
$$\int_{K} \mu_{t+1}(x) dx = \int_{K} \mu_t(x) dx \cdot\left( 1 - \gamma  \cdot \left\langle  \tilde \nabla_t, \frac{\int_{K} x \mu_t(x) dx}{\int_{K} \mu_t(x) dx} - \tilde x_t\right\rangle \right)$$
Note that this is indeed an ``approximate density'', since for $\gamma = 1/(3LD)$
\begin{equation}\label{eq:approx-centr-2}
    1 - \frac{\delta}{3} \leq \frac{\int_{K} \mu_{t+1}(x) dx}{\int_{K} \mu_t(x) dx} \leq 1 + \frac{\delta}{3}
\end{equation}
Setting $\delta = 1/T$ in Equation~\eqref{eq:approx-centr-2} and telescoping for $\int_{K} \mu_t(x) dx$ we get:
\begin{equation}\label{eq:bound-pot-apx-centr}
\frac{1}{e} \leq \int_{\rpledit{K}} \mu_{t+1}(x) dx \leq e
\end{equation}
Finally, the only thing that these derivations change with respect to the regret proof of Theorem~\ref{thm:log_concave_density} is that instead of having $\Phi_{T+1} \leq 1$ now we can only guarantee that $\Phi_{T+1} \leq e$. This only affects the constants in the final regret bound. %

\paragraph{Practical implementation. } \rpledit{Above we showed that in theory, our algorithm can be implemented in $\poly(d,T)$ running time. Even using the best available log-concave sampling techniques available today and ignoring the issue of an approximate centroid, the algorithm still requires $O(d^2)$ calls to compute the centroid with each call costing $O(dT)$ in a total of $O(d^3 T)$ to compute each $x_t$. This is prohibitively expensive for practical applications.

We remark, however, that there are techniques that can be applied to improve the \emph{practical} running time. To address the dependency on $T$, one may subsample the term in Equations \eqref{eq:log_density_f} and \eqref{eq:grad_log_density_f}. In certain problems like $\crocs$ (Section \ref{sec:crocs}) one may try to apply a dimensionality reduction like the Johnson-Lindenstrauss transform and solve the problem for a lower $d$. We also remark that it is often the case that for some functions, the sampler algorithm requires much less iterations than the provable bounds.  In an online companion\footnote{\url{https://gist.github.com/renatoppl/6086184ce5d5a49c617337e98b08afc8}} we provide a Python implementation of the Log-Concave Density algorithm for $\croco$ using the MALA algorithm of \cite{dwivedi2019log}. } 

\subsection{Implication: Online Convex Optimization with Subgradient Feedback}

In online convex optimization with subgradient feedback, the setting is identical to $\croco$ with the only difference being that now we do not require all functions $f_t$ to share a minimizer $\xst$.\footnote{The reason we need a common minimizer in $\croco$ is to guarantee that the terms $f_t(x_t) - f_t(\xst) \geq 0$ for all $t$, which is essential when approximating $1-x$ by $\exp(-\alpha_1 x)$. If we drop the common minimizer assumption, we can still have a weaker guarantee known as \emph{pseudo-regret}.} Instead, the comparator $\xst$ is now taken as a point in $\text{argmin}_{x \in K} \sum_{t \in [T]} f_t(x)$, but each $f_t$ may be minimized at a different point.

Algorithm~\ref{alg:logconcave} is also well-defined for (standard) online convex optimization with subgradient feedback. However, there is a small difference in how the regret is defined, which makes this result slightly weaker than the results of Section~\ref{sec:croco}; instead of comparing against $\sum_{t \in [T]} f_t(\xst)$, our algorithm now compares against a slightly inflated benchmark $(1+\eps) \sum_{t \in [T]} f_t(\xst)$ assuming $f_t(x) \geq 0, \forall x \in K, t \in [T]$ and $\eps > 0$.

\begin{proposition}\label{thm:oco_log_concave_density}Let $f_t:K \rightarrow \R$ be non-negative convex functions for all $t \in [T]$. With $\gamma = \eps/(3L)$, the regret of the Algorithm~\ref{alg:logconcave} for online convex optimization with subgradient feedback has the following pseudo-regret guarantee:
$$\sum_{t \in [T]} f_t(x_t) - (1+O(\eps))  \sum_{t \in [T]} f_t(\xst) \leq C + d L\log(T/L)$$
\end{proposition}

\noindent \proof{{\bf Proof.}}
We proceed as in the proof of Theorem \ref{thm:croco_log_concave_density}, up to the point where we establish that: $\Phi_t \geq \Phi_{t-1}(1+\gamma(f_t(x_t) - f_t(\xst) - \eps_t - Lr))$. Now we observe that $1-x \geq e^{-\alpha_{1,\eps}, x}$ and  $1+x \geq e^{\alpha_{2,\eps} x}$ for $x \in [0,\eps]$ for $\alpha_{1,\eps} = -\log(1-\eps)/\eps$ and $\alpha_{2,\eps} = \log(1+\eps)/\eps$. Using the Taylor expansion of $\log(1+x)$, we get: 
$ 1-O(\eps) \leq \alpha_{2,\eps} \leq 1 \leq  \alpha_{1,\eps} \leq 1+O(\eps)$ and hence $\alpha_{1,\eps} / \alpha_{2,\eps} \leq 1+O(\eps)$. Using that we bound the potential as follows:
$$\Phi_t \geq \Phi_{t-1} \cdot (1+\gamma f_t(x_t)  (1-\gamma( f_t(\xst) + \eps_t + Lr)) \geq \Phi_{t-1} \exp(\alpha_{2,\eps} \gamma f_t(x_t) - \alpha_{1,\eps} \gamma(f_t(\xst) + \eps_t + Lr))  $$
Telescoping and using the fact that densities integrate to at most $1$ we have:
$$\textstyle 1 \geq \Phi_{T+1} \geq \Phi_1\exp\left(\alpha_{2,\eps} \gamma \sum_{t \in [T]} f_t(x_t)  - \alpha_{1,\eps} \gamma \left(\sum_{t \in [T]} (f_t(\xst) + \eps_t) + TLr\right)\right)$$
Taking logarithms on both sides, re-arranging terms, and using the bound of $\Phi_1$ from Theorem \ref{thm:croco_log_concave_density}, we have:
$$\sum_{t \in [T]} f_t(x_t)-  (1+O(\eps)) \sum_{t \in [T]} f_t(\xst)  \leq O \left(\sum_{t \in [T]} \eps_t + TLr - \log(\Phi_1)/\gamma \right) = O(C + 1 + L d \log(LT)) \hspace{10pt} \Halmos$$
\endproof

\subsection{Application to $\crocs$ for the Symmetric Loss}\label{sec:crocs}

We show next that Algorithm~\ref{alg:logconcave} can be used for learning in $\crocs$ with the \rpledit{symmetric} loss and obtains the following regret guarantee. 

\begin{corollary}[Regret of the Log-Concave Density Algorithm]\label{thm:log_concave_density}
For $\crocs$ with the \rpledit{symmetric} loss, the regret of the Log-Concave Density Algorithm is $O(C_1 + d \log T)$, where $C_1 = \sum_{t \in [T]} |z_t|$ is the total amount of corruption which is unknown to the algorithm.
\end{corollary}

{The proof of the corollary is deferred to the Appendix.}

\section{A \texorpdfstring{$O(C_0 + d \log (1/\eps))$}- Algorithm for the \texorpdfstring{$\eps$}--Ball Loss}\label{sec:eps-ball}

We next shift our attention to algorithms for $\crocs$ with the $\eps$-ball loss. Since this loss is non-convex, we cannot use the $\croco$ framework. Instead, we will develop a customized solution to the $\eps$-ball loss, while still using the idea of keeping track of a density function; i.e., our algorithm works by keeping track of a density function $\mu_t : \B(0,1) \rightarrow \R$ that evolves from round to round. Initially, we set $\mu_1$ to be the uniform density over $\B(0, 1)$, i.e., $\mu_1(x) = 1/\Vol(\B(0,1))$ for all $x \in \B(0,1)$. %

\subsection{First Attempt: Using the Standard Median}

We start by describing a natural algorithm which ---although not the algorithm we ultimately analyze--- will be useful for providing intuition. This algorithm is as follows: once the context $u_t$ arrives, we compute the median $y_t$ of $\mu_t$ ``\emph{in the direction $u_t$}''. 
\begin{definition}[Median of a Distribution]\label{def:median-distr}
There are two equivalent ways to define the median of an \rpledit{atomless} distribution in a certain direction.
\begin{enumerate}
    \item Define a random variable $Z = \dot{X}{u_t}$, where $X$ is drawn from a density $\mu_t$. Then, $y_t$ is called the \emph{median of $Z$} if: $\P[Z \geq y_t] = \P[Z \leq y_t]$.
    \item $y_t \in \R$ is the \emph{median} of distribution $f$ if: $\int \mu_t(x) \1\{\dot{u_t}{x} \geq y_t\} dx = \int \mu_t(x) \1\{\dot{u_t}{x} \leq y_t\} dx$.
\end{enumerate}
\end{definition}
Note that since all the distributions that we work with in this work are derived from continuous density functions, they do not have point masses. Hence, the median (and later the $\eps$-window-median) is always well defined.

After we query $y_t$, we receive the feedback of whether $\yst_t =\dot{u_t}{\thetast} + z_t$ is larger or smaller than $y_t$. We do not know the amount of corruption added, but if we believe that it is more likely that this feedback is uncorrupted than corrupted, then we can try to increase the density whenever $\sigma_t (\dot{u_t}{x} - y_t)\geq 0$. For example, we could define:
$$\mu_{t+1}(x) = \left\{ \begin{aligned}
& 3/2 \cdot \mu_t(x), & & \text{if } \sigma_t (\dot{u_t}{x} - y_t)\geq 0 \\
& 1/2 \cdot \mu_t(x), & & \text{if } \sigma_t (\dot{u_t}{x} - y_t)< 0 \\
\end{aligned}\right.$$
Note that since $y_t$ is chosen to be median, then $\mu_{t+1}$ is still a density. 
\begin{lemma}
Function $\mu_t(\cdot)$ is a valid probability density function for all rounds $t$.
\end{lemma}

\noindent \proof{{\bf Proof.}}
We proceed with induction. For the base case and by the definition of $\mu_1(\cdot)$ to be a uniform density, the lemma holds. Assume now that $\mu_t(\cdot)$ is a valid probability density for some round $t = n$, i.e., $\int_{\B(0,1)} f_n(x) dx = 1$. Then, for round $t+1 = n+1$: 
\begin{align*}
\int_{\B} \mu_{t+1}(x) dx &= \frac{3}{2}\int_{\B} \mu_t(x) \1\{\sigma_t(\dot{u_t}{x} -  y_t) \geq 0\} dx + \frac{1}{2}\int_{\B} \mu_t(x) \1\{\sigma_t(\dot{u_t}{x} - y_t) < 0\} dx \\
&= \int_{\B} \mu_t(x) \1\{\sigma_t(\dot{u_t}{x} -  y_t) \geq 0\} dx + \frac{1}{2} \int_{\B} \mu_t(x) dx &\tag{grouping terms} \\ 
&= \int_{\B} \mu_t(x) \1\{\sigma_t(\dot{u_t}{x} -  y_t) \geq 0\} dx + \frac{1}{2} &\tag{inductive hypothesis}\\
&= \frac{1}{2} \cdot 2 \cdot \int_{\B} \mu_t(x) \1\{\sigma_t(\dot{u_t}{x} -  y_t) \geq 0\} dx + \frac{1}{2} \\ 
&= \frac{1}{2}\int_{\B} \mu_t(x) \1\{\sigma_t(\dot{u_t}{x} -  y_t) \geq 0\} dx + \frac{1}{2} \int_{\B} \mu_t(x) \1\{\sigma_t(\dot{u_t}{x} - y_t) < 0\} dx + \frac{1}{2} \\
&= \frac{1}{2} + \frac{1}{2} = 1
\end{align*}
where the penultimate equality is due to the definition of $y_t$ being the median of distribution $\mu_t(\cdot)$ (Definition~\ref{def:median-distr}). $\Halmos$
\endproof

Ideally, we would like the mass of the density around $\thetast$ to increase in all uncorrupted rounds. With this update rule, however, this is impossible to argue. To see why, observe that if the hyperplane $\{x \in \R^d; \dot{u_t}{x} = y_t\}$ is far from $\thetast$ then the total density in a ball $\B(\thetast, \eps)$ will increase. However, if the hyperplane intersects the ball $\B(\thetast, \eps)$, then some part of its density will increase and some will decrease. Since the density is non-uniform in the ball, we cannot argue that it will increase in good rounds, i.e., rounds where $y_t \in \B(\thetast, \eps)$.

\subsection{The \texorpdfstring{$\eps$}--Window Median Algorithm}

To address the issue above, we define the notion of the \emph{$\eps$-window median}. 

\begin{definition}[$\eps$-Window Median]\label{def:eps-win-med}
Given a random variable $Z$ taking values in $\R$ we say that an $\eps$-window median of $Z$ is a value $y$ such that: $\P [Z \leq y - {\eps}/{2}] = \P[Z \geq y \rpledit{+} {\eps}/{2} ].$
\end{definition}
We can also define the $\eps$-window median for a density $\mu_t(\cdot)$ as follows.
\begin{definition}[$\eps$-Window Median for Densities]
Given a density $\mu_t$ and a direction $u_t \in \B(0,1)$, we say that the $\eps$-window median of $\mu_t$ in the direction $u_t$ is the $\eps$-window median of a variable $Z = \dot{u_t}{X}$, where $X$ is drawn from a distribution with density $\mu_t$. Equivalently, this is the value $y_t \in \R$ such that:
\[\int_{\B} \mu_t(x) \1\{\dot{u_t}{x} \geq y_t + \eps/2\} dx = \int_{\B} \mu_t(x) \1\{\dot{u_t}{x} \leq y_t - \eps/2\} dx\]
\end{definition}

\begin{algorithm}[htbp]
\caption{$\textsc{$\eps$-Window Median Algorithm}$}
\label{alg:eps-win-med}
\DontPrintSemicolon
\SetAlgoNoLine
Initialize $\mu_1(x)$ to be the uniform density over $\B(0,1)$. \;
\For{rounds $t \in [T]$}{
    Observe context $u_t$.\;
    Query $\eps$-window median of $\mu_t$: $y_t$. \;
    Receive feedback $\sigma_t$ and update the density as:
    \[\mu_{t+1}(x) = \left\{ \begin{aligned}
     & 3/2 \cdot \mu_t(x), & & \text{if } \sigma_t \cdot (\dot{u_t}{x} - y_t)\geq \eps/2 \\
     & 1 \cdot \mu_t(x), & & \text{if } -\eps/2 \leq \sigma_t \cdot (\dot{u_t}{x} - y_t) \leq \eps/2 \\
     & 1/2 \cdot \mu_t(x), & & \text{if } \sigma_t \cdot (\dot{u_t}{x} - y_t) \leq -\eps/2 \\
    \end{aligned}\right.\]
}
\end{algorithm}

We first prove that $\mu_{t+1}(x)$ \rpledit{as defined in Algorithm \ref{alg:eps-win-med}} is a valid density.

\begin{lemma}\label{lem:density-eps-win-med}
Function $\mu_t(\cdot)$ is a valid probability density function for all rounds $t$. 
\end{lemma}

{We prove the lemma via a simple induction, and we defer the proof to the Appendix.}

We are now left to bound the regret of Algorithm~\ref{alg:eps-win-med}.

\begin{theorem}[Regret of $\eps$-Window Median]
The regret of the $\eps$-Window Median Algorithm is $O(C_0 + d \log(1/\eps))$.
\end{theorem}

\noindent \proof{{\bf Proof.}}
We define a potential function:
$$\Phi_t = \int_{\B(\thetast, \eps/2)} \mu_t(x) dx$$
For each round $t$, we distinguish the following three cases.

For \emph{Case 1}, if round $t$ is a corrupted round, then the potential decreases by at most a factor of $2$, i.e., $\Phi_{t+1} \geq \Phi_t/2$. This is because regardless of the feedback $\sigma_t$: $\mu_{t+1}(x) \geq (1/2) \mu_t(x)$ for all $x$. Note that there are at most $C_0$ such corrupted rounds.

For \emph{Case 2}, assume that round $t$ is an uncorrupted round in which we pick up a loss of $1$. In this case, note that the potential increases by a factor of $3/2$, i.e., $\Phi_{t+1} = (3/2) \Phi_t$. To see this, note that since we pick up a loss of $1$, then by definition the distance from $\thetast$ to the hyperplane $\{x; \dot{u_t}{x} = y_t\}$ has to be at least $\eps$. As a consequence, the ball $\B(\thetast, \eps/2)$ has to be inside the halfspace $\{x; \sigma_t(\dot{u_t}{x} - y_t) \geq \eps/2\}$ and therefore, $\mu_{t+1}(x) = (3/2) \mu_t(x)$ for all $x \in \B(\thetast, \eps/2)$. We denote by $L$ the total number of such uncorrupted rounds. Note that this $L$ corresponds also the total loss suffered through these rounds.

For \emph{Case 3}, assume that $t$ is an uncorrupted round in which we incur a loss of $0$. In that case, observe that the potential does not decrease, i.e., $\Phi_{t+1} \geq \Phi_t$. Indeed, since the round is uncorrupted, it must be the case that $\sigma_t (\dot{u_t}{\thetast}-y_t) \geq 0$. Therefore, for all $x \in \B(\thetast, \eps/2)$ we must have:  $\sigma_t (\dot{u_t}{x}-y_t) \geq -\eps/2$. Hence, $\mu_{t+1}(x) \geq \mu_t(x)$ for all $x \in \B(\thetast, \eps/2)$.

Putting it all together and telescoping for $\Phi_t$ we obtain:
$$\Phi_{T+1} \geq \Phi_1 \cdot \left(\frac{1}{2}\right)^{C_0} \cdot \left(\frac{3}{2}\right)^{L}$$
Since $f_{t}$ is always a density (Lemma~\ref{lem:density-eps-win-med}), we have that $\Phi_{T+1} \leq 1$. So, taking logarithms for both sides of the above equation, we get:
$$0 \geq \log \Phi_1 + C_0 \log \frac{1}{2} + L \log \frac{3}{2}$$
Reorganizing the terms:
$$L \leq O(C_0 - \log (\Vol(\B(\thetast, \eps))) = O(C_0 + d \log (1/\eps))$$
Finally, note that the regret from corrupted rounds is at most $C_0$ and the regret from uncorrupted rounds is $L$, so $\Reg \leq C_0 + L \leq O(C_0 + d \log(1/\eps)). \hspace{10pt} \Halmos$
\endproof

\paragraph{Relation to Multiplicative Weights.} Like the traditional MWU algorithm, we keep a weight over the set of candidate solutions and update it multiplicatively. However, it is worth pointing out some important key differences. First, unlike in experts' or bandits' settings, we do not get to observe the loss (not even an unbiased estimator thereof). We can only observe binary feedback, so it is impossible to update proportionally to the loss in each round. Second, we do not choose an action proportionally to the weights like MWU or EXP3. Instead, we use the $\eps$-window-median. \rpledit{In some sense, our algorithm resembles a (soft) policy elimination algorithm; we maintain a set of ``hypotheses'' (pertaining to the true $\theta^{\star}$). For those hypotheses within the uncertainty bound (i.e., our $\eps$-window), we keep their weight as is. Instead, the hypotheses that clearly violate the feedback received are downgraded exponentially.}

We conclude this section by discussing the running time of the $\eps$-window-median algorithm.

\begin{lemma}[Running Time]\label{lem:eps-win-med-runtime}
Algorithm~\ref{alg:eps-win-med} has runtime $O\left(T^d \cdot \poly(d,T)\right)$. 
\end{lemma}

\noindent \proof{{\bf Proof.}}
The running time in each step is dominated by the computation of the $\eps$-window median. Given an oracle that for each $u \in \R^d$ and $y \in \R$ returns the integral $\int \mu_t(x) \1\{ \dot{u}{x} \leq y\} dx$, we can use binary search to determine the $\eps$-window median. Observe that the function: $$\psi(y) = \frac{\int \mu_t(x) \1\{ \dot{u}{x} \leq y - \eps/2\} dx}{\int \mu_t(x) \1\{ \dot{u}{x} \geq y + \eps/2\} dx} $$
is monotonically increasing and computing the  $\eps$-window median is equivalent to finding a value of $y$ such that $\psi(y)=1$. Note also that the analysis does not require us to query the $\eps$-window median exactly. Rather, any point $y$ with $\psi(y) \in [1-\eps, 1+\eps]$ would lead to the same bound with a change only in the constants. 

Next, we discuss how to design an oracle to compute the integral $\int \mu_t(x) \1\{ \dot{u}{x} \leq y\} dx$. Note that $\mu_t$ is piecewise constant, where each piece is one the regions in space determined by the hyperplanes $\{x: \dot{u_t}{x} = y_t$\}. The maximum number of regions created by $T$ hyperplanes in $\R^d$ is given by the Whitney number, which is at most $O(T^d)$ (\cite{stanley2004introduction}). Keeping track of each of these regions explicitly leads to an $O(T^d \poly(d,T))$ algorithm for computing the integral. $\Halmos$
\endproof

\rpledit{The runtime of $O(T^d)$ can be rather impractical, unless the dimension $d$ is small. The main merit of the $\eps$-window median algorithm is to achieve the optimal regret for $\crocs$ (and not focus on the runtime). That said, our algorithm not only achieves the optimal regret, but it does so with a runtime that is faster than that of the previous state-of-the-art~\citep{krishnamurthy2023contextual}. We leave as open problem whether it is possible to obtain the same guarantee using  a $\poly(d,T)$ algorithm.}

\section{Discussion}

In this paper, we studied learning in contextual search settings while being robust and agnostic to adversarial noise. Traditionally, contextual search settings focus on two loss functions: (i) the $\eps$-ball loss, and (ii) the absolute loss. For the $\eps$-ball loss, we introduced an algorithm with regret $O(C + d \log (1/\eps)$, thus significantly improving over the previously known bound of $O(d^3 \log (1/\eps) \log^2(T) + C \log (T) \log(1/\eps))$ of \citet{krishnamurthy2023contextual}. Based on prior work, our regret guarantee is \emph{tight}. For the absolute loss, we provided an efficient algorithm with regret $O(C + d \log T)$. Aside from the optimized regret guarantees, our techniques represent a significant contribution, as they depart from prior approaches in contextual search settings. Specifically, instead of the traditional view of maintaining a set of vectors that is consistent with the feedback that we have received so far, we keep track of carefully constructed density functions over the original set of target vectors.

En route to obtaining the efficient algorithm for the absolute loss, we studied a more general setting: learning in convex optimization settings while being robust and agnostic to adversarial noise and receiving only subgradient-type feedback from the adversary; we call this the $\croco$ setting. The $\croco$ setting (and the approach that we take to tackle it) can be of independent interest beyond contextual search, especially given the implications of our algorithm for obtaining approximate pseudo-regret guarantees for standard online convex optimization with subgradient feedback.

There are several avenues for future research stemming from our work. In terms of the $\crocs$ setting with absolute loss, the question of obtaining \emph{optimal} regret bounds while being agnostic to adversarial noise remains open: \emph{Is it possible for the densities-based approach to obtain regret $O(C + d)$ in the absolute loss?} Our density-based approach may also be useful in establishing optimized regret bounds for $\crocs$ with the \emph{pricing loss}, although it seems that such a result would require significant new machinery; the best known regret bound for the pricing loss in the corruption setting currently is $O(C \log^2 (T) + d^3 \log^3(T))$~\citep{krishnamurthy2023contextual}, which is significantly bigger than the lower bound of $\Omega(d \log \log T)$ of the noiseless setting. Note that the algorithmic approach of \citet{krishnamurthy2023contextual} worked because after a fixed set of rounds, the set of remaining vectors consistent with the feedback received thus far had converged to a ball of radius $\eps$ around the target vector. From that point onward, using the smallest price within this $\eps$ ball guaranteed that the extra regret picked up was $\eps T$. In our density-based approach, however, the knowledge set of vectors never changes; we just shift probability mass between the points. As a result, we cannot guarantee that after a fixed number of rounds we will know how to provide the lowest price that will guarantee only an $\eps T$ loss. Finally, in terms of $\croco$, it would be interesting to see if the density-based approach could be tightened in order to obtain no-regret bounds, rather than no-approximate-regret ones. 

\bibliographystyle{plainnat}
\bibliography{refs, refs2}

\newpage
{
\section*{Appendix}

In this section, we include 2 proofs which were omitted from the main body of the paper.

\noindent \proof{{\bf Proof of Corollary~\ref{thm:log_concave_density}.}
Observe that the contextual search problem with the \rpledit{symmetric} loss $\ell(y,y^*) = \abs{y - y^*}$ can be viewed as an instance of $\croco$ where $f_t(x) = \abs{\langle u_t, x- \thetast\rangle}$.  Using the notation in Section \ref{sec:model_contextual_search}, we observe that if $\sigma_t \in \{-1,+1\}$ is the feedback in the contextual search problem then $\tilde \nabla_t = - \sigma_t u_t$ is a corrupted gradient oracle for $\croco$ with corruption level $C_1$. To see that, define $\tilde f_t(x) = \abs{\langle u_t, x- \thetast\rangle - z_t}$ where $z_t$ is the corruption level introduced by the adversary defined in Section \ref{sec:model_contextual_search}. With that, $\tilde \nabla_t$ is a valid subgradient for $\tilde f_t$, i.e.,
$\tilde f_t(x) \geq \tilde f_t(x_t) + \langle\tilde \nabla_t, x - x_t \rangle $. Now, observe that $\abs{f_t(x) - \tilde f_t(x)} \leq z_t, \forall x$. Combining these, we get:
$$ f_t(x) \geq  f_t(x_t) + \left \langle\tilde \nabla_t, x - x_t \right \rangle - 2 z_t $$
which is the definition of a corrupted oracle in Equation \eqref{eq:corrupted_oracle}. Finally, observe that while all $f_t$'s are different functions, $\thetast$ is a common minimizer for all $t \in [T]$. The Lipschitz constant is bounded by  $L = 1$ since $\norm{u_t} \leq 1$ and the diameter is bounded by $D=2$ since $\norm{\thetast} \leq 1$. Applying Theorem \ref{thm:croco_log_concave_density} directly to this setting gives us the desired $O(C_1 + d \log T)$ bound.
\endproof
}

{
\noindent \proof{{\bf Proof of Lemma~\ref{lem:density-eps-win-med}.}}
We prove this lemma by induction. For the base case, note that by definition the lemma holds for $t = 1$, since $f_1(x)$ is the uniform density over $\B(0,1)$. Assume now that $f_t(x)$ is a density for some $t = n$, i.e., $\int_\B f_t(x) dx = 1$. Then, for round $t+1 = n+1$ we define the following sets: 
\begin{align*}
U_+ &= \left\{ x \in \B(0,1): \sigma_t (\dot{u_t}{x} - y_t)\geq \eps/2 \right\} \\
U_0 &= \left\{ x \in \B(0,1): -\eps/2 \leq \sigma_t (\dot{u_t}{x} - y_t) \leq \eps/2\right \} \\
U_- &= \left\{ x \in \B(0,1): \sigma_t (\dot{u_t}{x} - y_t) \leq -\eps/2\right\}
\end{align*}
As for $f_{t+1}(x)$ we have:
\begin{align*}
    \int_\B f_{t+1}(x) dx &= \frac{3}{2}\int_{U_+} f_t(x)dx + \int_{U_0} f_t(x)dx + \frac{1}{2}\int_{U_-} f_t(x) dx \\
    &= \int_{U_+} f_t(x)dx + \frac{1}{2} \int_{U_0} f_t(x)dx + \frac{1}{2} \int_\B f_t(x) dx &\tag{grouping terms} \\ 
    &= \int_{U_+} f_t(x)dx + \frac{1}{2} \int_{U_0} f_t(x)dx + \frac{1}{2} &\tag{inductive hypothesis} \\
    &= \frac{1}{2} + \frac{1}{2} = 1
\end{align*}
where the penultimate inequality is due to the following property which is direct from the definition of the $\eps$-window median (Definition~\ref{def:eps-win-med}):
\[
\int_{U_+} f_t(x) dx= \int_{U_-} f_t(x) dx = \frac{1}{2}-\frac{1}{2} \int_{U_0} f_t(x) dx
\]
This concludes our proof. $\Halmos$
\endproof
}

\end{document}